\documentclass{article}

\usepackage{times}

\usepackage{amssymb,amsmath}
\usepackage{amsthm}
\usepackage{graphicx}
\usepackage{xcolor}

\theoremstyle{plain}
\newtheorem{theorem}{Theorem}
\newtheorem{proposition}{Proposition}

\theoremstyle{definition}
\newtheorem{definition}{Definition}
\newtheorem{remark}{Remark}
\newtheorem{example}{Example}

\usepackage{graphicx} 
\graphicspath{{images/}}

\usepackage{subfigure}

\usepackage{natbib}

\usepackage{algorithm}
\usepackage[noend]{algorithmic}

\usepackage{hyperref}

\usepackage[accepted]{icml2015}

\usepackage{mindgap}
\usepackage{mindgap_fig}

\icmltitlerunning{\icmlourtitle}

\begin{document}
\twocolumn[
\icmltitle{\icmlourtitle}


\icmlauthor{Olivier Fercoq}{olivier.fercoq@telecom-paristech.fr}
\icmlauthor{Alexandre Gramfort}{alexandre.gramfort@telecom-paristech.fr}
\icmlauthor{Joseph Salmon}{joseph.salmon@telecom-paristech.fr}
\icmladdress{Institut Mines-T\'el\'ecom,
T\'el\'ecom ParisTech, CNRS LTCI\\
46 rue Barrault, 75013, Paris, France\\}

\icmlkeywords{Safe rules, Convex Optimization, Duality gap, Lasso, Screening Test
}

\vskip 0.3in
]

\begin{abstract}
Screening rules allow to early discard irrelevant variables from the
optimization in Lasso problems, or its derivatives, making solvers faster. In
this paper, we propose new versions of the so-called \textit{safe rules} for the
Lasso. Based on duality gap considerations, our new rules create safe
test regions whose diameters converge to zero, provided that one relies on
a converging solver. This property helps screening out more variables, for a
wider range of regularization parameter values. In addition to faster
convergence, we prove that we correctly identify the active sets (supports) of
the solutions in finite time. While our proposed strategy can cope with any
solver,  its performance is demonstrated using a coordinate descent algorithm
particularly adapted to machine learning use cases. Significant computing time
reductions are obtained with respect to previous safe rules.
\end{abstract}


\section{Introduction}
Since the mid 1990's, high dimensional statistics has attracted considerable
attention, especially in the context of linear regression with more
explanatory variables than observations: the so-called $p>n$ case. In such a
context, the least squares with $\ell_1$ regularization, referred to as the
Lasso \cite{Tibshirani96} in statistics, or Basis Pursuit
\cite{Chen_Donoho_Saunders98} in signal processing, has been one of the most
popular tools. It enjoys theoretical guarantees \cite{Bickel_Ritov_Tsybakov09},
as well as practical benefits: it provides sparse solutions and fast convex
solvers are available. This has made the Lasso a popular method in modern
data-science tool\-kits. Among successful fields where it has been applied, one
can mention dictionary learning \cite{Mairal},
bio-statistics~\cite{Haury_Mordelet_Verra-Licona_Vert12}
and medical imaging~\cite{Lustig_Donoho_Pauly07,Gramfort_Kowalski_Hamalainen12}
to name a few.

Many algorithms exist to approximate Lasso solutions, but it is still a burning
issue to accelerate solvers in high dimensions. Indeed,
although some other variable selection and prediction methods exist
\cite{Fan_Lv2008}, the best performing methods usually rely on the Lasso.
For stability selection methods
\cite{Meinshausen_Buhlmann10,Bach08b,Varoquaux_Thirion_Gramfort12}, hundreds
of Lasso problems need to be solved. For non-convex approaches such as SCAD
\cite{Fan_Li01} or MCP \cite{Zhang10}, solving the Lasso is often a
required preliminary step \cite{Zou06,Zhang_Zhang12,Candes_Wakin_Boyd08}.

Among possible algorithmic candidates for solving the Lasso, one can
mention homotopy methods \cite{Osborne_Presnnell_Turlach00}, LARS
\cite{Efron_Hastie_Johnstone_Tibshirani04}, and approximate homotopy
\cite{Mairal_Yu12}, that provide solutions for the full Lasso
path, \ie for all possible choices of tuning parameter $\lambda$. More
recently, particularly for $p>n$, coordinate descent
approaches \cite{Friedman_Hastie_Hofling_Tibshirani07} have proved to be among
the best methods to tackle large scale problems.

Following the seminal work by \citet{ElGhaoui_Viallon_Rabbani12}, screening
techniques have emerged as a way to exploit the known sparsity of the solution
by discarding features prior to starting a Lasso solver.
Such techniques are coined \emph{safe rules} when they screen out
coefficients guaranteed to be zero in the targeted optimal solution. Zeroing those
coefficients allows to focus more precisely on the non-zero ones (likely to
represent signal) and helps reducing the computational burden. We refer to
\cite{Xiang_Wang_Ramadge14} for a concise introduction on
safe rules. Other alternatives have tried to screen the Lasso
relaxing the ``safety''. Potentially, some variables are wrongly disregarded
and post-processing is needed to recover them. This is for instance the
strategy adopted for the \textit{strong rules}
\cite{Tibshirani_Bien_Friedman_Hastie_Simon_Tibshirani12}.

The original basic safe rules operate as follows: one chooses a fixed tuning
parameter $\lambda$, and before launching any solver, tests whether a
coordinate can be zeroed or not (equivalently if the corresponding variable
can be disregarded or not). We will refer to such safe rules as
\textit{static safe rules}. Note that the test is performed according to a safe
region, \ie a region containing a dual optimal solution of the Lasso problem.
In the static case, the screening is performed only once, prior any optimization iteration. Two directions have emerged to improve on static strategies.

\begin{itemize}
   \item
The first direction is oriented towards the resolution of the
Lasso for a large number of tuning parameters. Indeed, practitioners commonly
compute the Lasso over a grid of parameters and select the
best one in a data-driven manner, \eg by cross-validation.
As two consecutive $\lambda's$ in the grid lead to similar solutions, knowing
the first solution may help improve screening for the second one. We call
\textit{sequential safe rules} such strategies, also referred to as recursive
safe rules in \cite{ElGhaoui_Viallon_Rabbani12}. This road has been pursued in
\cite{Wang_Zhou_Wonka_Ye13,Xu_Ramadge13,Xiang_Wang_Ramadge14}, and can be thought of as a
``warm start'' of the screening (in addition to the warm start
of the solution itself). When performing sequential safe rules, one should keep in mind that generally, only an approximation of the previous dual solution is computed. Though, the safety of the rule is guaranteed only
if one uses the exact solution. Neglecting this issue, leads to
``unsafe'' rules:  relevant variables might be wrongly disregarded.
  \item
The second direction aims at improving the screening by
interlacing it throughout the optimization algorithm itself: although screening
might be useless at the beginning of the algorithm, it might become (more) efficient as the algorithm proceeds towards the optimal solution. We
call these strategies \textit{dynamic safe rules} following \cite{Bonnefoy_Emiya_Ralaivola_Gribonval14,Bonnefoy_Emiya_Ralaivola_Gribonval15}.
\end{itemize}
Based on convex optimization arguments, we leverage duality gap
computations to propose a simple strategy unifying both sequential
and dynamic safe rules. We coined GAP SAFE rules such safe rules.

The main contributions of this paper are 1) the introduction of new safe rules
which demonstrate a clear practical improvement compared to prior strategies
2) the definition of a theoretical framework for comparing safe
rules by looking at the convergence of their associated safe regions.

In Section~\ref{sec:safe_rules}, we present the framework and the basic
concepts which guarantee the soundness of static and dynamic screening rules.
Then, in Section~\ref{sec:new_contributions}, we introduce the new concept of
converging safe rules. Such rules identify in finite time the active variables
of the optimal solution (or equivalently the inactive variables),
and the tests become more and more precise as the optimization
algorithm proceeds. We also show that our new GAP SAFE rules, built on dual
gap computations, are converging safe rules since their associated safe regions
have a diameter converging to zero. We also explain how our GAP SAFE tests
are sequential by nature. Application of our GAP SAFE rules with a coordinate
descent solver for the Lasso problem is proposed in
Section~\ref{sec:experiments}. Using standard data-sets, we report the time improvement compared to prior safe rules.

\subsection{Model and notation}
We denote by $[d]$ the set $\{1, \ldots, d\}$
for any integer $d\in\bbN$. Our observation vector is $y \in \bbR^n$ and the
design matrix $X= [x_1,\cdots,x_p ] \in \bbR^{n\times p}$ has $p$ explanatory
variables (or features) column-wise. We aim at approximating $y$ as a linear combination of few variables $x_j$'s,
hence expressing $y$ as $X \beta$ where $\beta \in \bbR^p$ is a sparse vector.
The standard Euclidean norm is written
$\|\cdot\|$, the $\ell_1$ norm $\|\cdot\|_1$, the $\ell_\infty$ norm
$\|\cdot\|_\infty$, and the matrix transposition of
a matrix $Q$ is denoted by ${Q}^\top$. We denote $(t)_+=\max(0,t)$.

For such a task, the Lasso is often considered (see \citet{Buhlmann_vandeGeer11} for an introduction).
For a tuning parameter $\lambda>0$, controlling the
trade-off between data fidelity and sparsity of the solutions, a Lasso
estimator $\tbeta{\lambda}$ is any solution of the primal optimization
problem
\begin{equation}\label{eq:Lasso}
\tbeta{\lambda} \in \argmin_{\beta \in \bbR^p}\underbrace{\frac{1}{2} \norm{X
\beta - y}_2^2 + \lambda \norm{\beta}_1}_{=P_\lambda(\beta)} \enspace .
\end{equation}
Denoting $\dual = \big\{ \theta \in \bbR^n \; : \;
\abs{x_j^\top \theta } \leq 1, \forall j \in [p] \big\}$ the dual feasible set,
a dual formulation of the Lasso reads
(see for instance \citet{Kim_Koh_Lustig_Boyd_Gorinevsky07} or
\citet{Xiang_Wang_Ramadge14}):
\begin{equation}\label{eq:dual_problem}
\ttheta{\lambda}=\argmax_{\theta \in \dual \subset \bbR^n}
\underbrace{\frac{1}{2}\norm{y}^2_2 - \frac{\lambda^2}{2}\norm{\theta -
\frac{y}{\lambda}}^2_2}_{=D_\lambda(\theta)}.
\end{equation}
We can reinterpret Eq.~\eqref{eq:dual_problem}
as $\ttheta{\lambda}=\Pi_{\dual}(y/\lambda)$, where $\Pi_{\mathcal{C}}$
refers to the projection onto a closed convex set $\mathcal{C}$. In particular,
this ensures that the dual solution
$\ttheta{\lambda}$ is always unique, contrarily to the primal
$\tbeta{\lambda}$.

\begin{table*}[t!]\centering
\small
\def\arraystretch{1.5}
\begin{tabular}{|c|c|c|c|}
\hline
Rule & Center & Radius & Ingredients  \\ \hline
Static  Safe \cite{ElGhaoui_Viallon_Rabbani12} & $y/\lambda $&
$\largeR[\frac{y}{\lambda_{\max}}]$ & $\lambda_{\max}=\|X^\top y
\|_\infty\!=\!|x_{j^\star}^\top y |$\\
 \hline
Dynamic ST3 \cite{Xiang_Xu_Ramadge11} & $y/\lambda-\delta x_{j^\star}\!$ &$
(\largeR[\theta_k]^{2}-\delta^2)^{\frac{1}{2}} $ & $\delta=
\left(\frac{\lambda_{\max}}{\lambda}-1\right)/\|x_{j^\star}\|^2$ \\\hline
Dynamic Safe \cite{Bonnefoy_Emiya_Ralaivola_Gribonval14}&$ y/\lambda$ &
$\largeR[\theta_k]$
& $\theta_k \in \dual$ (\eg as in \eqref{eq:thetak} ) \\ \hline
Sequential \cite{Wang_Zhou_Wonka_Ye13} & $\ttheta{\lambda_{t-1}}$&
$\left|\frac{1}{\lambda_{t-1}}-\frac{1}{\lambda_{t}}\right| \|y\|$ &
exact $\ttheta{\lambda_{t-1}}$ required\\ \hline
GAP SAFE sphere (proposed)&$ \theta_k$ & $\bestR[\lambda_t]{\beta_k}{\theta_k}=
\frac{1}{\lambda_t}\sqrt{2 G_{\lambda_t}(\beta_k,\theta_k)}
$ & dual gap for
$\beta_k,\theta_k$
\\ \hline
\end{tabular}
\caption{Review of some common safe sphere tests.}
\label{table:spheres}
\end{table*}

\subsection{A KKT detour}
For the Lasso problem, a primal solution $\tbeta{\lambda}\in\bbR^p$
and the dual solution $\ttheta{\lambda}\in\bbR^n$
are linked through the relation:
\begin{equation}\label{eq:primal_dual}
y= X\tbeta{\lambda}+\lambda \ttheta{\lambda} \enspace .
\end{equation}
The Karush-Khun-Tucker (KKT) conditions state:
\begin{equation}\label{eq:KKT}
\forall j \in [p], \; x_j^{\top} \ttheta{\lambda} \in
\begin{cases}
\{\sign(\tbeta{\lambda}_j)\} & \, \text{if}\quad \tbeta{\lambda}_j \neq 0, \\
[-1,1]                       & \, \text{if}\quad\tbeta{\lambda}_j=0.\\
\end{cases}
\end{equation}
See for instance \cite{Xiang_Wang_Ramadge14} for more details.
The KKT conditions lead to the fact that
for $\lambda \geq \lambda_{\max}=\|X^\top y\|_\infty$,
$0\in \bbR^p$ is a primal solution. It can be considered as the mother of
all safe screening rules. So from now on, we assume that $\lambda \leq
\lambda_{\max}$ for all the considered $\lambda$'s.

\section{Safe rules}\label{sec:safe_rules}
Safe rules exploit the KKT condition \eqref{eq:KKT}.
This equation implies that $\tbeta{\lambda}_j=0$ as soon as
$|x_j^{\top} \ttheta{\lambda}| < 1$. The main challenge is that the
dual optimal solution is unknown. Hence, a safe rule aims at
constructing a set $\mathcal{C} \subset \bbR^n$ containing $\ttheta{\lambda}$.
We call such a set $\mathcal{C}$ a \emph{safe region}. Safe regions are all the
more helpful that for many $j$'s, $\mu_{\mathcal{C}}(x_j):=
\sup_{\theta \in \mathcal{C}}|x_j^{\top} \theta |<1$, hence
for many $j$'s, $\tbeta{\lambda}_j=0$.

Practical benefits are obtained if one can construct a region
$\mathcal{C}$ for which it is easy to compute
its \emph{support function}, denoted by
$\sigma_{\mathcal{C}}$ and defined for any $x \in \bbR^n$ by:
\begin{equation}
 \sigma_{\mathcal{C}}(x)=\max_{\theta \in \mathcal{C}}~ x^{\top}\theta \, .
\end{equation}
Cast differently, for any safe region $\mathcal{C}$,
any $j \in [p]$, and any primal optimal solution $\tbeta{\lambda}$, the
following holds true:
\begin{equation}\label{eq:sphere_test}
\text{If } \mu_{\mathcal{C}}(x_j) =
\max(\sigma_{\mathcal{C}}(x_j),\sigma_{\mathcal{ C}} (-x_j) ) <1 \text{ then }
\tbeta{\lambda}_j \!= 0.
\end{equation}
We call \emph{safe test} or \emph{safe rule}, a test associated to $\mathcal{C}$
and screening out explanatory variables thanks to
Eq.~\eqref{eq:sphere_test}.

\begin{remark}\label{rem:convex_hull}
Reminding that the support function of a set is the same as the support function
of its closed convex hull
\cite{Hiriart-Urruty_Lemarechal93}[Proposition~V.2.2.1], we restrict
our search to closed convex safe regions.
\end{remark}

Based on a safe region $\mathcal{C}$ one can partition the explanatory variables
into a safe active set $A^{\lambda}(\mathcal{C})$ and a safe zero set
$Z^{\lambda}(\mathcal{C})$ where:
\begin{align}
 A^{(\lambda)}(\mathcal{C}) = \{j \in [p] : \mu_{\mathcal{C}}(x_j) \geq 1\},
\label{eq:active_set}\\
 Z^{(\lambda)}(\mathcal{C}) = \{j \in [p] : \mu_{\mathcal{C}}(x_j) < 1\}.
\end{align}
Note that for nested safe regions $\mathcal{C}_1 \subset
\mathcal{C}_2 $ then
$A^{(\lambda)}(\mathcal{C}_1) \subset A^{(\lambda)}(\mathcal{C}_2)$.
Consequently, a natural goal is to find safe regions as small as possible:
narrowing safe regions can only  increase the number of screened out variables.

\begin{remark}\label{rem:equicorrelation_set}
If $\mathcal{C}=\{\ttheta{\lambda}\}$,
the safe active set is the equicorrelation set
$A^{(\lambda)}(\mathcal{C})=\mathcal{E}_{\lambda}:=\{j \in [p] :
|x_j^\top\ttheta{\lambda}|= 1  \}$
(in most cases \cite{Tibshirani13} it is exactly the support of
$\tbeta{\lambda}$). Even when the Lasso is
not unique, the equicorrelation set contains all the solutions' supports.
The other extreme case is when $\mathcal{C}=\dual$,
and $A^{(\lambda)}(\mathcal{C})=[p]$. Here, no variable is
screened out: $Z^{(\lambda)}(\mathcal{C})= \emptyset$ and the screening is
useless.
\end{remark}

We now consider common safe regions whose support functions are easy to obtain
in closed form. For simplicity we focus only on balls and domes, though
more complicated regions could be investigated \cite{Xiang_Wang_Ramadge14}.

\subsection{Sphere tests}
Following previous work on safe rules, we call \emph{sphere tests}, tests
relying on balls as safe regions. For a sphere test, one chooses a ball
containing $\ttheta{\lambda}$ with center $c$ and radius $r$, \ie
$\mathcal{C}=B(c,r)$. Due to their simplicity, safe spheres have been the
most commonly investigated safe regions (see for instance Table
\ref{table:spheres} for a brief review). The corresponding test is
defined as follows:
\begin{equation}\label{ineq:sphere_test}
\text{If }  \mu_{B(c,r)}(x_j)=|x_j^\top c|+r \|x_j\|<1, \text{ then }
\tbeta{\lambda}_j=0.
\end{equation}
Note that for a fixed center, the smaller the radius, the better the
safe screening strategy.

\begin{example}\label{ex:elghaoui}
The first introduced sphere test \cite{ElGhaoui_Viallon_Rabbani12} consists
in using the center $c=y/\lambda$ and radius $r=|1/\lambda
-1/\lambda_{\max}|\|y\|$. Given that $\ttheta{\lambda}=\Pi_{\dual}(y/\lambda)$,
this is a safe region since $y/\lambda_{\max} \in \dual$ and
$\|y/\lambda_{\max} - \Pi_{\dual}(y/\lambda)\| \leq \|y\|   |1/\lambda
-1 / \lambda_{\max}|$. However, one can check that this static safe rule is
useless as soon as
\begin{equation}\label{eq:lambda_critic_static}
\frac{\lambda}{\lambda_{\max}} \leq \min_{j \in [p]} \left(\frac{1 +
|x_j^\top y | / (\|x_j\| \|y\|)}{1+\lambda_{\max}/(\|x_j\| \|y\|)} \right).
\end{equation}
\end{example}

\subsection{Dome tests}
Other popular safe regions are \emph{domes}, the intersection
between a ball and a half-space. This kind of safe region has been considered
for instance in
\cite{ElGhaoui_Viallon_Rabbani12,Xiang_Ramadge11,Xiang_Wang_Ramadge14,
Bonnefoy_Emiya_Ralaivola_Gribonval15}.
We denote $D(\domecenter,\domeradius,\domeprop,\domenormal)$ the dome with
ball center $\domecenter$, ball radius $\domeradius$,
oriented hyperplane with unit normal vector $\domenormal$ and
parameter $\domeprop$ such that $c-\domeprop r \domenormal$ is the
projection of $c$ on the hyperplane (see Figure \ref{fig:dome} for an
illustration in the interesting case $\alpha>0$).
\begin{remark}
The dome is non-trivial whenever $\alpha \in [-1,1]$. When $\alpha=0$, one gets
simply a hemisphere.
\end{remark}

For the dome test one needs to compute the support function
for $\mathcal{C}=D(\domecenter,\domeradius,\domeprop,\domenormal)$.
Interestingly, as for balls, it can be obtained in a closed form. Due to its
length though, the formula is deferred to the Appendix (see also
\cite{Xiang_Wang_Ramadge14}[Lemma
3] for more details).

\begin{figure}
\centering
\includegraphics[width = 0.7\linewidth]{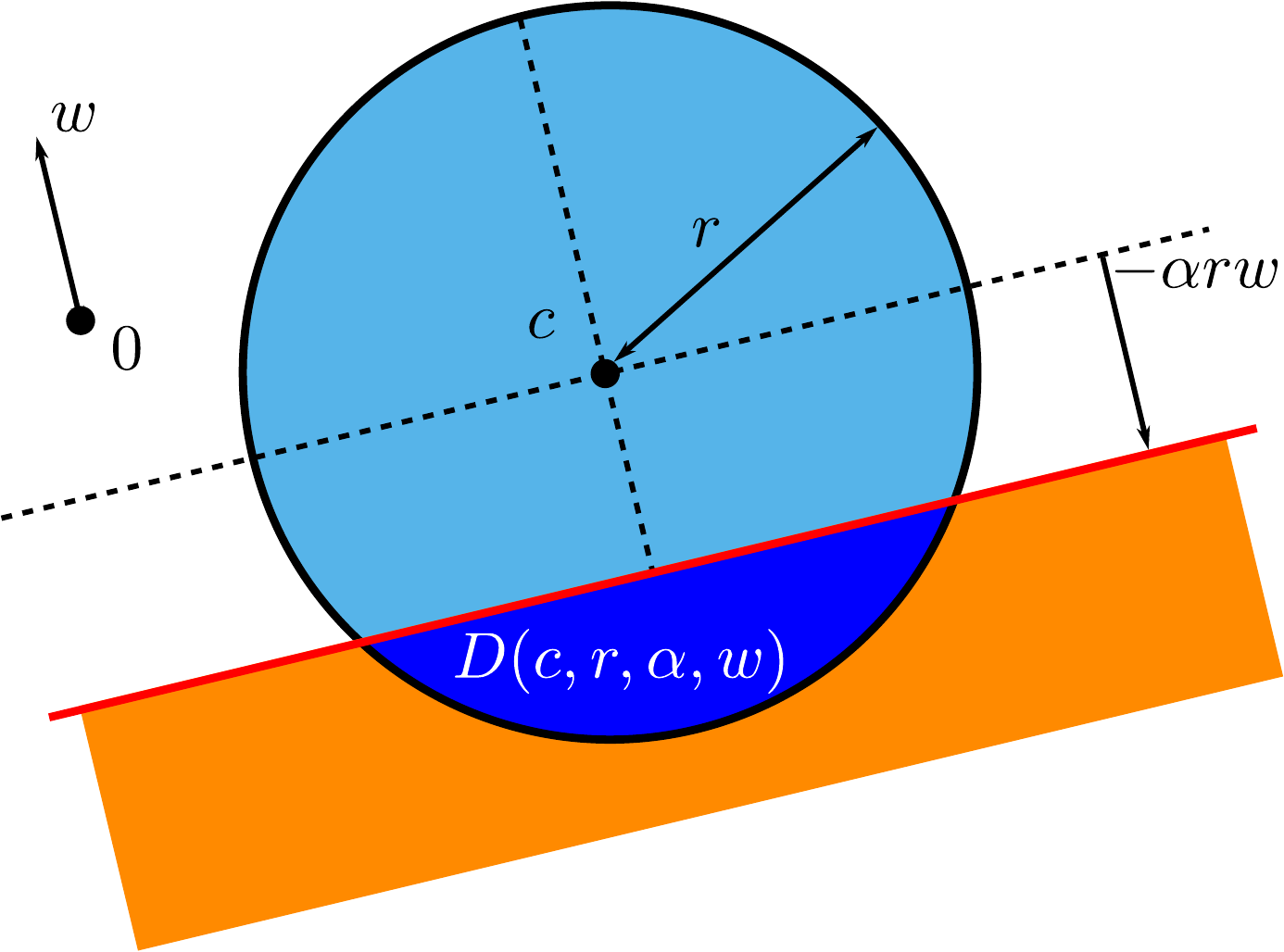}
\caption{Representation of the dome
$D(\domecenter,\domeradius,\domeprop,\domenormal)$ (dark blue). In our case,
note that $\alpha$ is positive.}
\label{fig:dome}
\end{figure}

\subsection{Dynamic safe rules} \label{sec:dynamic_safe_rules}
For approximating a solution $\tbeta{\lambda}$ of the
Lasso primal problem $P_{\lambda}$, iterative algorithms are commonly used.
We denote $\beta_k \in \bbR^p$ the current estimate after $k$ iterations of any
iterative algorithm (see Section \ref{sec:experiments} for a specific study on
coordinate descent). Dynamic safe rules aim at discovering safe regions that
become narrower as $k$ increases. To do so, one first needs dual feasible
points: $\theta_k \in \dual$.
Following \citet{ElGhaoui_Viallon_Rabbani12} (see also
\cite{Bonnefoy_Emiya_Ralaivola_Gribonval14}), this can be achieved by a simple
transformation of the current residuals $\rho_k = y - X \beta_k$, defining
$\theta_k$ as
\begin{equation}\label{eq:thetak}
\begin{cases}
\theta_k \!=\! {\alpha}_k \rho_k,\\
 \alpha_k\! = \!\min \!\Big[\!\max \!\left(\frac{y^\top \!\rho_k}{\lambda
\norm{\rho_k}^2},\!\frac{-1}{\norm{X^\top
\rho_k}_{\infty}}\right)\!\!,\!\frac{1}{\norm{X^\top
\rho_k}_{\infty}}\Big].\\
\end{cases}
\end{equation}
Such dual feasible $\theta_k$ is proportional to $\rho_k$, and is the
closest point (for
the norm $\|\cdot\|$) to $y/\lambda$ in $\dual$ with such a property, \ie
$\theta_k=\Pi_{\dual
\cap \Span(\rho_k)}(y/\lambda)$. A reason for choosing this dual point
is that the dual optimal solution $\ttheta{\lambda}$ is the projection of
$y/\lambda$ on
the dual feasible set $\dual$, and  the optimal $\ttheta{\lambda}$
is proportional to $y - X \tbeta{\lambda}$, \lcf
Equation~\eqref{eq:primal_dual}.

\begin{remark}
\label{rem:convergence}
Note that if $\lim_{k \to +\infty} \beta_k=\tbeta{\lambda}$ (convergence of the
primal) then with the previous display and \eqref{eq:primal_dual}, we can show
that
 $\lim_{k \to +\infty} \theta_k=\ttheta{\lambda}$.
Moreover, the convergence of the primal is unaltered by
safe rules: screening out unnecessary coefficients of $\beta_k$, can only decrease the distance between $\beta_k$ and $\tbeta{\lambda}$.
\end{remark}

\begin{example}\label{ex:dynamic}
Note that any
dual feasible point $\theta\in\dual$ immediately provides a ball that
contains $\ttheta{\lambda}$ since
\begin{equation}\label{ineq:largeR}
\norm{\ttheta{\lambda} - \frac{y}{\lambda}} = \min_{\theta' \in \dual}
\norm{\theta' - \frac{y}{\lambda}}
\leq  \norm{\theta - \frac{y}{\lambda}} := \largeR[\theta].
 \end{equation}
The ball $B\big(y/\lambda, \largeR[\theta_k]\big)$ corresponds to the simplest
safe region introduced in~\cite{Bonnefoy_Emiya_Ralaivola_Gribonval14,
Bonnefoy_Emiya_Ralaivola_Gribonval15} (\lcf Figure~\ref{fig:couronne1} for
more insights).
When the algorithm proceeds, one expects that $\theta_k$ gets closer to
$\ttheta{\lambda}$, so $\|\theta_k - y/\lambda\|$
should get closer to $\|\ttheta{\lambda} - y/\lambda\|$. Similarly to Example
\ref{ex:elghaoui}, this dynamic rule becomes useless once
$\lambda$ is too small. More
precisely, this occurs as soon as
\begin{equation}\label{eq:lambda_critic_dynamic_hard}
\frac{\lambda}{\lambda_{\max}} \leq \min_{j \in [p]} \left(\frac{1 +
|x_j^\top y | / (\|x_j\| \|y\|)}{ \lambda_{\max} \|\ttheta{\lambda}\|/\|y\|\\
+\lambda_{\max}/(\|x_j\| \|y\|)} \right).
\end{equation}
Noticing that $\|\ttheta{\lambda}\|\leq \|y/\lambda\|$ (since $\Pi_{\dual}$ is a
contraction and $0\in\dual$) and proceeding as for
\eqref{eq:lambda_critic_static}, one can show that
this dynamic safe rule is inefficient when:
\begin{equation}\label{eq:lambda_critic_dynamic}
\frac{\lambda}{\lambda_{\max}} \leq \min_{j \in [p]}
\left(
\frac{|x_j^\top y |}
{\lambda_{\max}}
\right).
\end{equation}
This is a critical threshold, yet the screening might
stop even at a larger $\lambda$ thanks to
Eq.~\eqref{eq:lambda_critic_dynamic_hard}. In practice
the bound in Eq.~\eqref{eq:lambda_critic_dynamic_hard} cannot be evaluated a
priori due to the term $\|\ttheta{\lambda}\|$). Note also that the bound
in Eq.~\eqref{eq:lambda_critic_dynamic} is close to the one in
Eq.~\eqref{eq:lambda_critic_static}, explaining the similar behavior observed
in our experiments (see Figure~\ref{fig:screening_proportion} for instance).
\end{example}




\begin{figure*}%
\centering
\subfigure[Location of the dual optimal $\ttheta{\lambda}$ in the
annulus.]{\includegraphics[width=.24\linewidth]{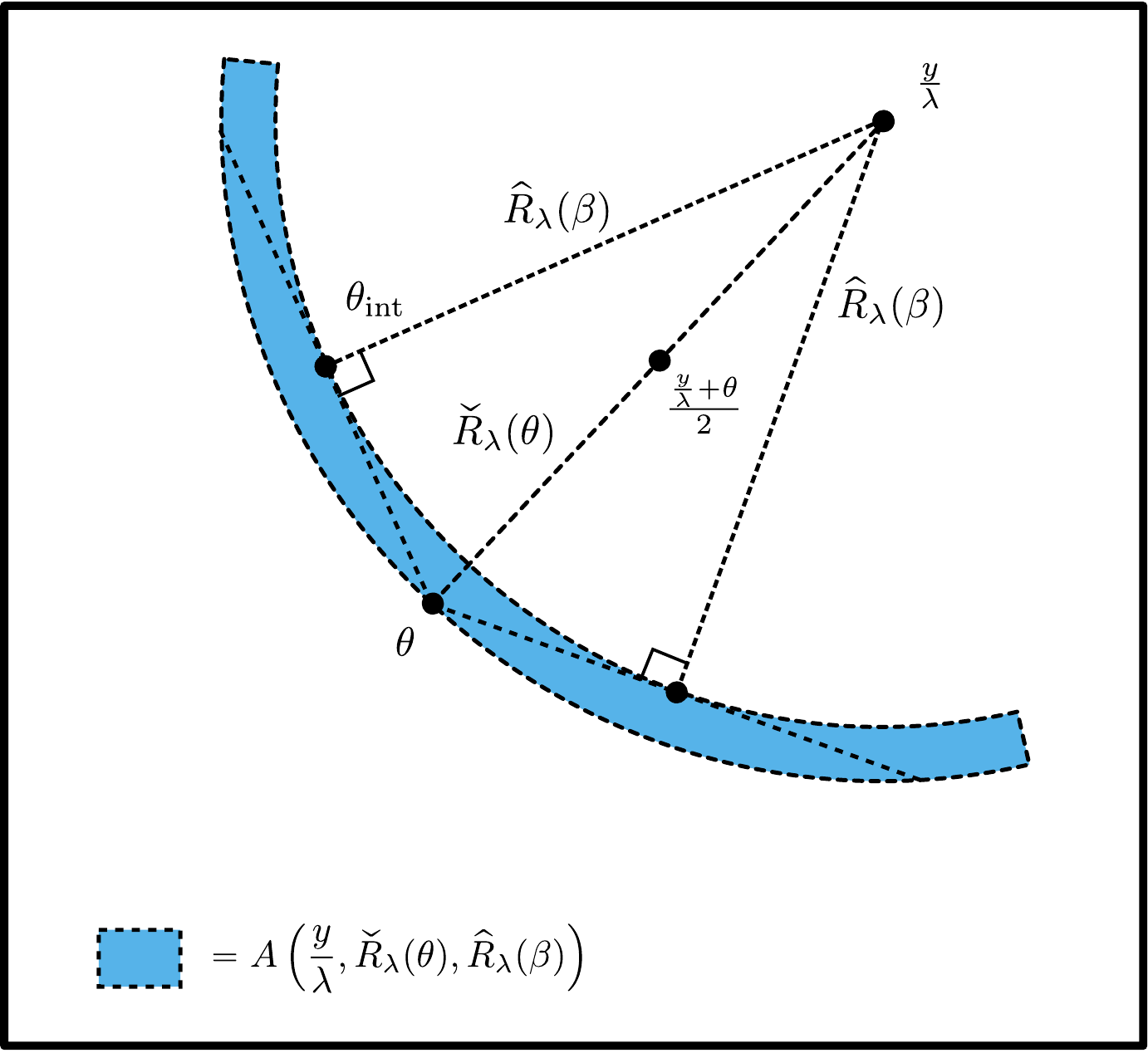}}
\subfigure[Refined location of the dual optimal $\ttheta{\lambda}$  (dark
blue).]{\includegraphics[width=.24\linewidth]{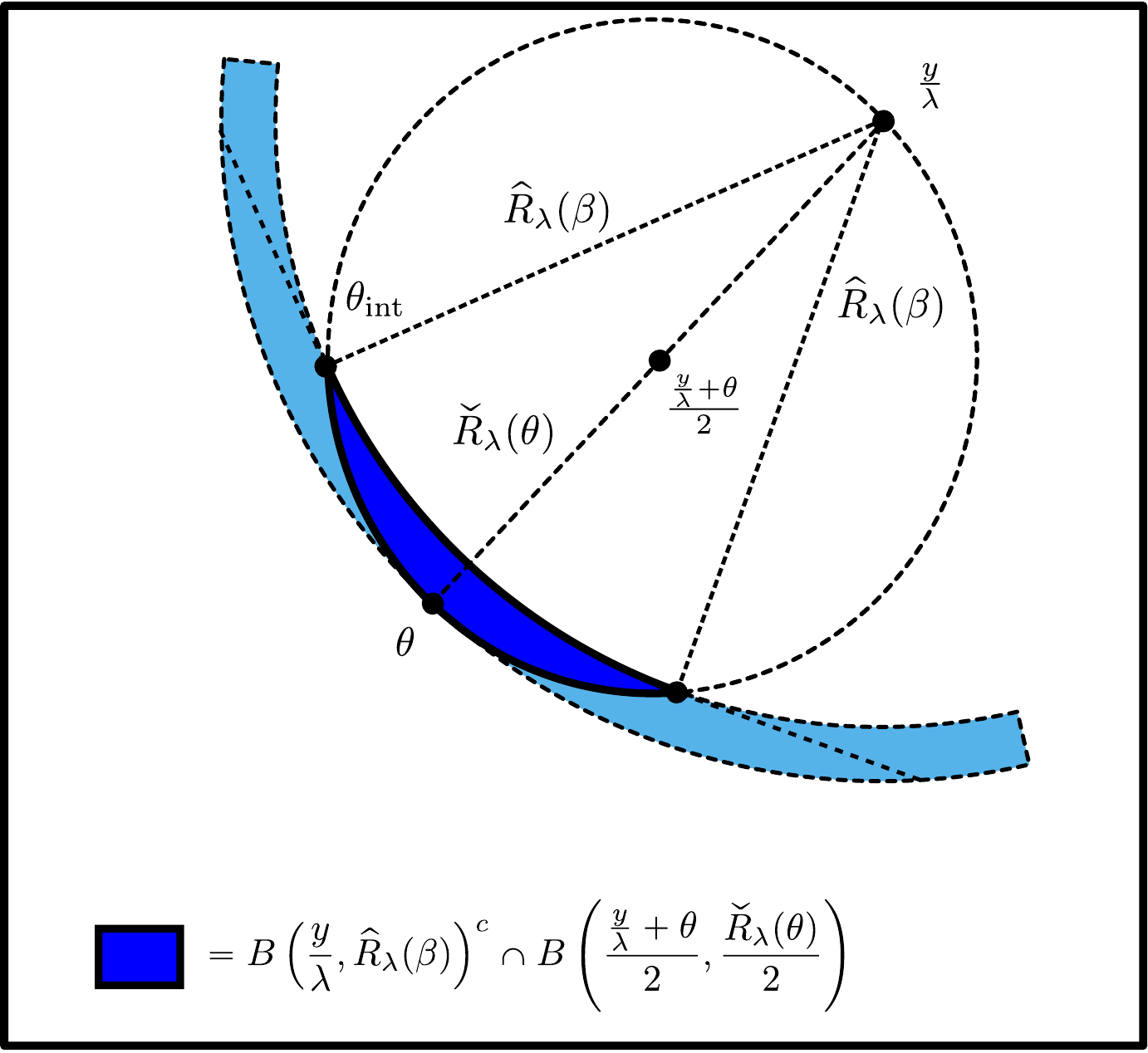}}
\subfigure[Proposed GAP SAFE sphere
(orange).]{\includegraphics[width=.24\linewidth]{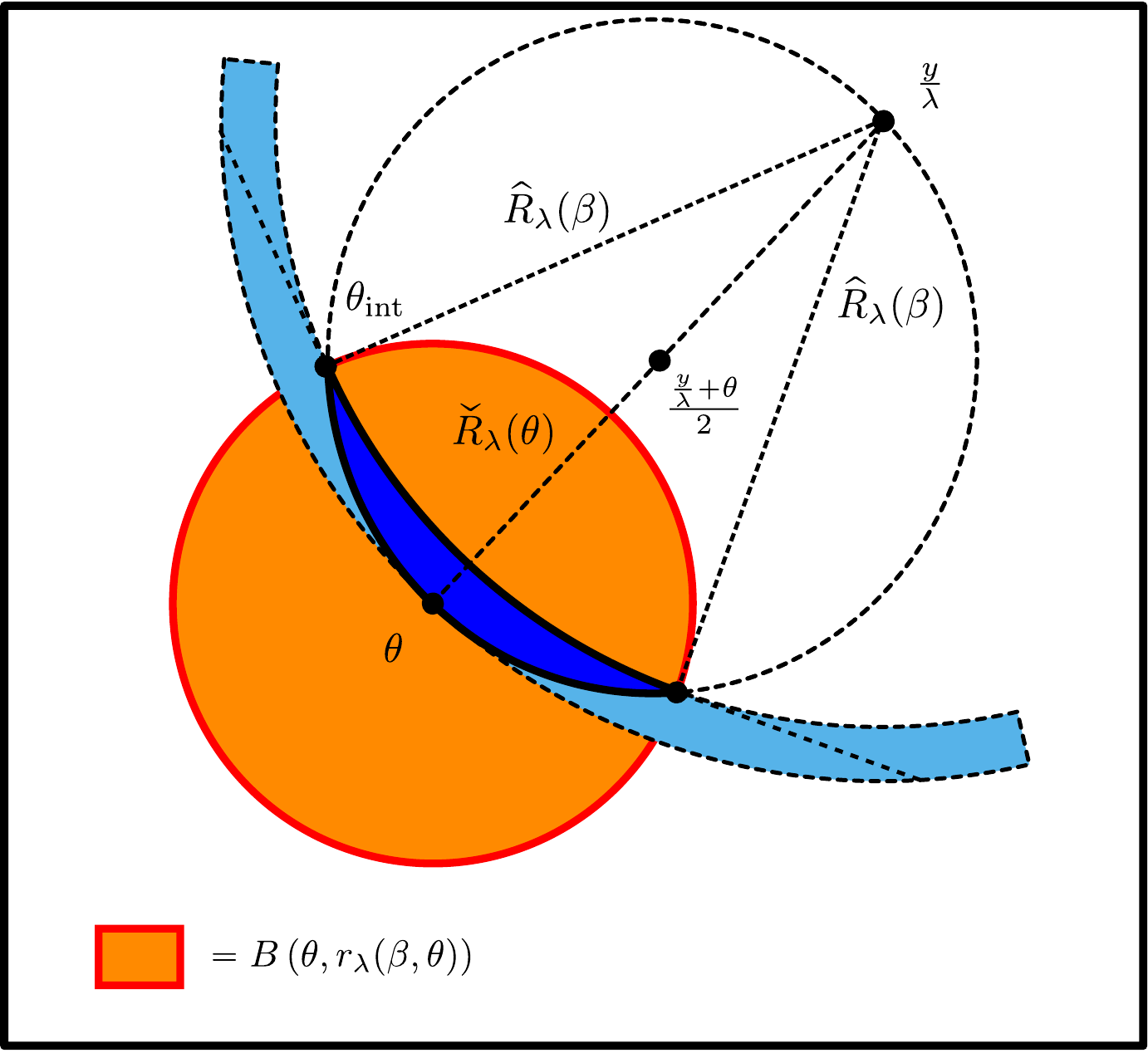}}
\subfigure[Proposed GAP SAFE dome
(orange).]{\includegraphics[width=.24\linewidth]{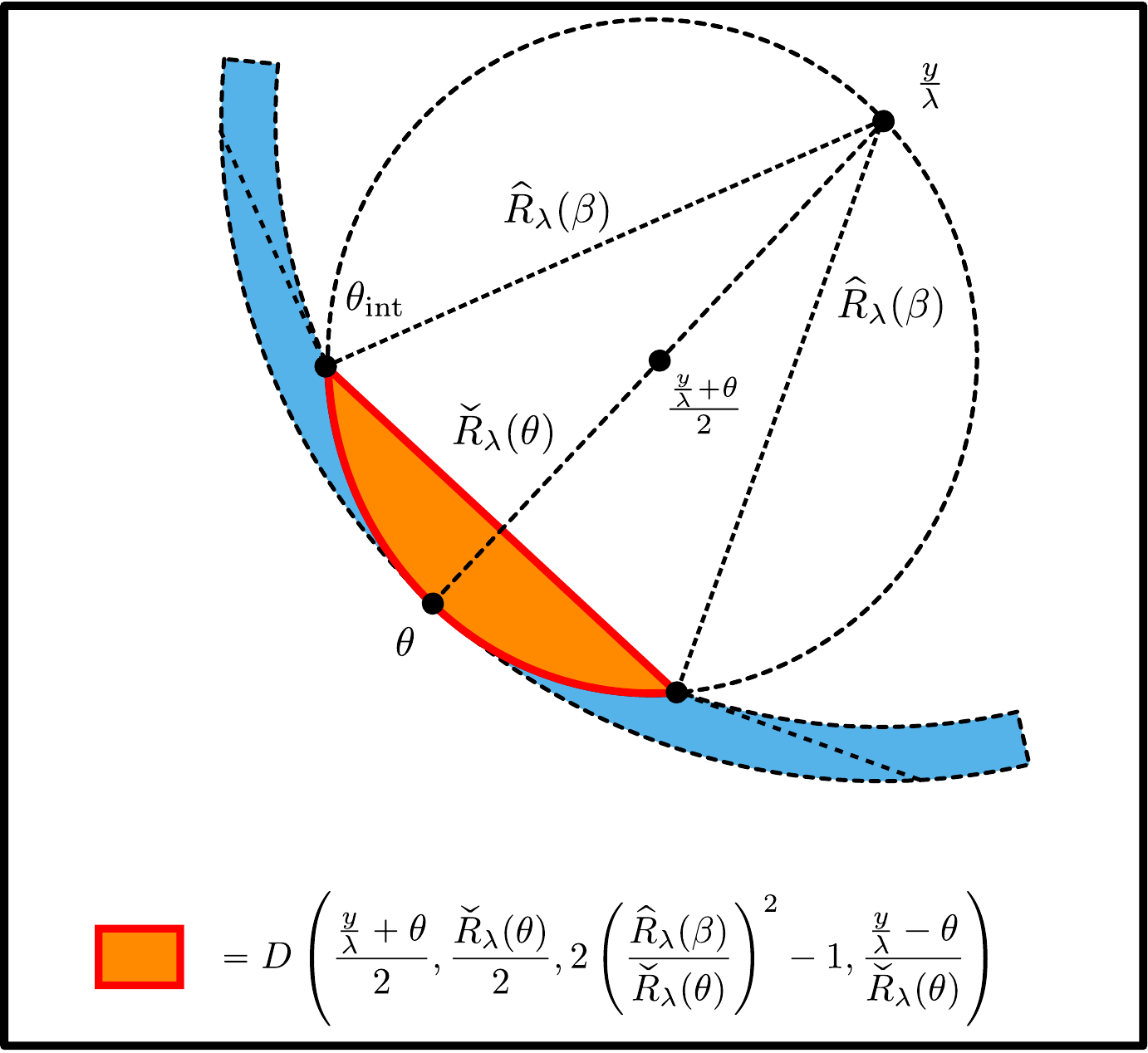}}
\caption{Our new GAP SAFE sphere and dome (in orange). The dual optimal
solution $\ttheta{\lambda}$ must lie in the dark blue region; $\beta$ is any
point in $\bbR^p$, and $\theta$ is any point in the dual feasible set $\dual$.
Remark that the GAP SAFE dome is included in the GAP SAFE sphere, and that it is
the convex hull of the dark blue region.}
\label{fig:couronne1}
\end{figure*}

\section{New contributions on safe rules}\label{sec:new_contributions}

\subsection{Support discovery in finite time}

Let us first introduce the notions of converging safe regions and converging
safe tests.

\begin{definition}
Let $(\mathcal{C}_{k})_{k \in \bbN}$ be a sequence of closed convex sets in
$\bbR^n$ containing $\ttheta{\lambda}$. It is a converging sequence of safe
regions for the Lasso with parameter $\lambda$
if the diameters of the sets converge to zero. The associated safe screening rules
are referred to as converging safe tests.
\end{definition}

Not only converging safe regions are crucial to speed up
computation, but they are also helpful to reach exact active set
identification in a finite number of steps.
More precisely, we prove that one recovers the equicorrelation set of the Lasso
(\lcf Remark~\ref{rem:equicorrelation_set}) in finite time with any converging
strategy: after a finite number of steps, the
equicorrelation set $\mathcal{E}_{\lambda}$ is exactly identified.
Such a property is sometimes referred to as finite identification of the
support \cite{Liang_Fadili_Peyere14}. This is summarized in the following.

\begin{theorem}\label{th:converging_safe}
Let $(\mathcal{C}_{k})_{k \in \bbN}$ be a sequence of converging safe
regions. The estimated support provided by $\mathcal{C}_{k}$,
$A^{(\lambda)}(\mathcal{C}_k)=\{ j \in [p] : \max_{\theta \in
\mathcal{C}_{k}} |\theta^\top x_j|  \geq 1 \}$,
satisfies  $\lim_{k \to \infty} A^{(\lambda)} (\mathcal{C}_k)
 = \mathcal{E}_{\lambda}$, and there exists $k_0 \in \bbN$ such that $\forall k
\geq k_0$ one gets $ A^{(\lambda)} (\mathcal{C}_k)
= \mathcal{E}_{\lambda}$.
\end{theorem}


\begin{proof}
The main idea of the proof is to use that
$\lim_{k\to\infty}\mathcal{C}_k=\{\ttheta{\lambda}\}$,
$\lim_{k\to\infty}\mu_{\mathcal{C}_k}(x)=\mu_{\{\ttheta{\lambda}\}}(x)=|x^\top
\ttheta{\lambda}|$ and that the set $A^{(\lambda)} (\mathcal{C}_k)$ is discrete.
Details are delayed to the Appendix.
\end{proof}

\begin{remark}
A more general result is proved for a specific algorithm (Forward-Backward) in
\citet{Liang_Fadili_Peyere14}.
Interestingly, our scheme is independent of the algorithm considered (\eg
Forward-Backward \cite{Beck_Teboulle09}, Primal Dual \cite{Chambolle_Pock11},
coordinate-descent \cite{Tseng01,Friedman_Hastie_Hofling_Tibshirani07})
and relies only on the convergence of a sequence of safe regions.
\end{remark}

\subsection{GAP SAFE regions: leveraging the duality gap}
In this section, we provide new dynamic safe rules built on converging safe
regions.
\begin{theorem}
\label{thm:finerball}
Let us take any $(\beta, \theta) \in \bbR^{p} \times \dual$.
Denote $\smallR[\beta]:=\frac{1}{\lambda}\big( \norm{y}^2 - \norm{X \beta - y}^2
- 2
\lambda \norm{\beta}_1 \big)_+^{1/2}, \largeR[\theta]:= \norm{\theta -
y/\lambda}, \ttheta{\lambda}$ the dual optimal Lasso solution and
$\bestRtilde[\lambda]{\beta}{
\theta}:=\sqrt{\largeR[\theta]^2-\smallR[\beta]^2}$,
then
\begin{equation}\label{eq:bestR}
 \ttheta{\lambda} \in B\Big(\theta,
\bestRtilde[\lambda]{\beta}{\theta}\Big).
\end{equation}
\end{theorem}

\begin{proof}

The construction of the ball $B(\theta,
\bestRtilde[\lambda]{\beta}{\theta})$ is based on the weak duality theorem (\lcf
\cite{Rockafellar_Wets98} for a reminder on weak and strong duality).
Fix $\theta\in \dual$ and $\beta\in \bbR^p$, then it holds that
\begin{equation*}
\frac{1}{2}\norm{y}^2 - \frac{\lambda^2}{2}\norm{\theta -
\frac{y}{\lambda}}^2 \leq \frac{1}{2}\norm{X \beta - y}^2 + \lambda
\norm{\beta}_1 .
\end{equation*}
Hence,
\begin{equation}\label{ineq:smallR_main}
\norm{\theta-\frac{y}{\lambda}}
\geq \! \frac{  \sqrt{\left(\norm{y}^2 - \norm{X \beta
- y}^2 - 2 \lambda \norm{\beta}_1\right)_+} }{\lambda}.
\end{equation}
In particular, this provides $\|\ttheta{\lambda}-y/\lambda\| \geq
\smallR[\beta]$.
Combining \eqref{ineq:largeR} and \eqref{ineq:smallR_main}, asserts
that $\ttheta{\lambda}$
belongs to the annulus
$A(y/\lambda, \largeR[\theta], \smallR[\beta]) :=
\{z \in \bbR^n : \smallR[\beta] \leq \|z-y/\lambda\| \leq \largeR[\theta] \}$
(the light blue zone in Figure~\ref{fig:couronne1}).

Remind that the
dual feasible set $\dual$ is convex, hence
$\dual \cap B(y/\lambda, \largeR[\theta])$ is also convex.
Thanks to \eqref{ineq:smallR_main}, $\dual \cap B(y/\lambda, \largeR[\theta]) =
\dual \cap A(y/\lambda, \largeR[\theta], \smallR[\beta])$,
and then $\dual \cap A(y/\lambda, \largeR[\theta], \smallR[\beta])$ is convex
too. Hence, $\ttheta{\lambda}$ is inside the annulus $A(y/\lambda,
\largeR[\theta], \smallR[\beta])$ and so is
$[\theta, \ttheta{\lambda}] \subseteq A(y/\lambda, \largeR[\theta],
\smallR[\beta])$ by convexity (see Figure~\ref{fig:couronne1},(a) and
Figure~\ref{fig:couronne1},(b)). Moreover,
$\ttheta{\lambda}$ is the point of $[\theta, \ttheta{\lambda}]$ which is closest
to $y/\lambda$. The farthest  where $\ttheta{\lambda}$ can be according to this
information would be if $[\theta, \ttheta{\lambda}]$
were tangent to the inner ball $B(y/\lambda, \smallR[\beta])$ and
$\|\ttheta{\lambda}-y/\lambda\|=\smallR[\beta]$.
Let us denote $\thetaint$ such a point.
The tangency property reads $\norm{\thetaint - y/\lambda}=\smallR[\beta]$ and
$(\theta-\thetaint)^\top (y/\lambda - \thetaint)= 0$. Hence, with the
later and the definition of $\largeR[\theta]$,
$\norm{\theta-y/\lambda}^2=\norm{\theta-\thetaint}^2 +
\norm{\thetaint-y/\lambda}^2
$and
$\norm{\theta-\thetaint}^2 =\largeR[\theta]^2-\smallR[\beta]^2
$.

Since by construction $\ttheta{\lambda}$ cannot be further away from
$\theta$ than $\thetaint$ (again, insights can be
gleaned from Figure~\ref{fig:couronne1}), we conclude that
$\ttheta{\lambda} \in B\big(\theta,
(\largeR[\theta]^2-\smallR[\beta]^2)^{1/2}\big)$.
\end{proof}

\begin{remark}
Choosing $\beta=0$ and $\theta= y/\lambda_{\max}$, then one recovers the static safe rule given in Example \ref{ex:elghaoui}.
\end{remark}

With the definition of the primal (resp. dual) objective for
$P_\lambda(\beta)$,
(resp. $D_\lambda(\theta)$), the duality gap reads as
$G_{\lambda}(\beta,\theta)=P_\lambda(\beta)-D_\lambda\left(\theta \right)$.
Remind that if $G_\lambda (\beta,\theta) \leq \epsilon$, then one has
$P_\lambda(\beta)-P_\lambda(\tbeta{\lambda})\leq \epsilon$, which is a
standard stopping criterion for Lasso solvers.
The next proposition establishes a connection between the radius
$\bestR[\lambda]{\beta}{\theta}$ and the duality gap $G_\lambda (\beta,\theta)$.
\begin{proposition}\label{prop:dualgapisradius}
For any $(\beta, \theta) \in \bbR^{p} \times \dual$, the following holds
\begin{equation}
\bestRtilde[\lambda]{\beta}{\theta}^2 \leq
\bestR[\lambda]{\beta}{\theta}^2 := \frac{2}{\lambda^2} G_\lambda
(\beta,\theta).
\end{equation}
\end{proposition}
\begin{proof}
Use the fact that
$\largeR[\theta]^2=  \norm{\theta - y/\lambda}^2$ and
$\smallR[\beta]^2 \geq  (\norm{y}^2- \norm{X \beta -y}^2 -
2\lambda\norm{\beta}_1)/\lambda^2$.
\end{proof}


If we could choose  the ``oracle'' $\theta=\ttheta{\lambda}$ and
$\beta=\tbeta{\lambda}$  in \eqref{eq:bestR} then
we would obtain a zero radius. Since those quantities are unknown, we rather
pick dynamically the current available estimates given by an optimization
algorithm:
$\beta=\beta_k$ and $\theta=\theta_k$ as in Eq.~\eqref{eq:thetak}.
Introducing GAP SAFE spheres and domes as below,
Proposition~\ref{prop:convergence} ensures that they are converging
safe regions.

\textbf{GAP SAFE sphere}:
\begin{equation}\label{eq:GAP_SAFE_SPHERE}
\mathcal{C}_k=
B\left(\theta_k,\bestR[\lambda]{\beta}{\theta}\right).
\end{equation}
\textbf{GAP SAFE dome}:
\begin{equation}\label{eq:GAP_SAFE_DOME}
\mathcal{C}_k=
D\!\!\left(\frac{\frac{y}{\lambda}+\theta_k}{2},\frac{\largeR[\theta_k]}{2},
2\left(\!\frac
{
\smallR[\beta_k]}{\largeR[\theta_k]} \right)^2\!\!\!-1,
\frac{\theta_k-\frac{y}{\lambda}}{\|\theta_k-\frac{y}{\lambda}\|}\right).
\end{equation}
\begin{proposition}\label{prop:convergence}
For any converging primal sequence $(\beta_k)_{k \in \bbN}$, and dual sequence
$(\theta_k)_{k \in \bbN}$
defined as in Eq.~\eqref{eq:thetak}, then the GAP SAFE sphere and the GAP SAFE
dome are converging safe regions.
\end{proposition}
\begin{proof}
For the GAP SAFE sphere the result follows from strong duality,
Remark~\ref{rem:convergence} and Proposition~\ref{prop:dualgapisradius} yield
$\lim_{k \to \infty}\bestR{\beta_k}{\theta_k}=0$, since $\lim_{k \to
\infty}\theta_k= \ttheta{\lambda}$ and $\lim_{k \to \infty}\beta_k=
\tbeta{\lambda}$. For the GAP SAFE dome, one can check that it
is included in the GAP SAFE sphere, therefore inherits the convergence (see
also Figure \ref{fig:couronne1},(c) and (d)).
\end{proof}

\begin{remark}\label{rem:ST3}
The radius $\bestR{\beta_k}{\theta_k}$ can be compared with the radius
considered for the Dynamic Safe  rule and Dynamic ST3
\cite{Bonnefoy_Emiya_Ralaivola_Gribonval14} respectively:
$\largeR[\theta_k]=\|\theta_k-y/\lambda\|^2$ and
$(\largeR[\theta_k]^2-\delta^2)^{1/2}$, where
$\delta=(\lambda_{\max}/\lambda-1)/\norm{x_{j^\star}}$.
We have proved that $\lim_{k\to \infty}\bestR{\beta_k}{\theta_k}=0$, but
a weaker property is satisfied by the two other radius:
$\lim_{k\to \infty}\largeR[\theta_k]= \largeR[\ttheta{\lambda}] =
\|\largeR[\ttheta{\lambda}]-y/\lambda\|^2$ and
$\lim_{k \to
\infty}(\largeR[\theta_k]^2-\delta^2)^{1/2}=(\largeR[\ttheta{\lambda}]
^2-\delta^2)^{1/2}>0$.
\end{remark}


\subsection{GAP SAFE rules : sequential for
free}\label{sec:sequential_safe_rules}

As a byproduct, our dynamic screening tests provide a warm
start strategy for the safe regions, making our  GAP SAFE rules inherently
sequential. The next proposition shows their efficiency when attacking a new
tuning parameter, after having solved the Lasso for a previous $\lambda$, even
only approximately. Handling approximate solutions is a critical issue to
produce safe sequential strategies: without taking into account the
approximation error, the screening might disregard relevant variables,
especially the one near the safe regions boundaries. Except for
$\lambda_{\max}$, it is unrealistic to assume that one can dispose of exact
solutions.

Consider $\lambda_{0}=\lambda_{\max}$ and a non-increasing sequence of
$T-1$ tuning parameters $(\lambda_t)_{t \in [T-1]}$ in $(0,\lambda_{\max})$.
In practice, we choose the common grid \cite{Buhlmann_vandeGeer11}[2.12.1]):
$\lambda_t=\lambda_{0} 10^{-\delta t/(T-1)}$ (for instance in
Figure~\ref{fig:screening_proportion}, we considered $\delta=3$).
The next result controls how the duality gap, or
equivalently, the diameter of our GAP SAFE regions, evolves from $\lambda_{t-1}$
to $\lambda_{t}$.

\begin{proposition}\label{prop:sequential}
Suppose that $t\geq 1$ and $(\beta,\theta) \in \bbR^p \times \dual$.
Reminding $\bestRsqr{\beta}{\theta}=2G_{\lambda_t} (\beta,\theta)/\lambda_t^2 $,
the following holds
\begin{align}
\bestRsqr{\beta}{\theta}&=
\left(\frac{\lambda_{t-1}}{\lambda_t} \right)
\bestRsqr[\lambda_{t-1}]{\beta}{\theta}\\
&+(1-\frac{\lambda_{t}}{\lambda_{t-1}})
 \norm{\frac{X \beta -y}{\lambda_t}}^2  - (\frac{\lambda_{t-1}}{\lambda_{t}}-1
)\norm{\theta}^2\nonumber .
\end{align}
\end{proposition}
\begin{proof}
Details are given in the Appendix.
\end{proof}


This proposition motivates to screen sequentially as follows: having
$(\beta, \theta) \in \bbR^p \times \dual$
such that $G_{\lambda_{t-1}} (\beta,\theta) \leq \epsilon$,
then, we can screen using the GAP SAFE sphere with center $\theta$ and radius
$\bestR{\beta}{\theta}
$. The adaptation to the GAP SAFE dome is straightforward and consists
in replacing $\theta_k,\beta_k,\lambda$ by $\theta,\beta,\lambda_t$ in the GAP
SAFE dome definition.

\begin{remark}
The basic sphere test of~\cite{Wang_Zhou_Wonka_Ye13} requires the
exact dual solution $\theta=\ttheta{\lambda_{t-1}}$ for center, and has radius
$|1/\lambda_{t}- 1/\lambda_{t-1}|\norm{y}$,
which is strictly larger than ours. Indeed, if one has access to dual
and primal optimal solutions at $\lambda_{t-1}$, \ie
$(\theta,\beta)=(\ttheta{\lambda_{t-1}},\tbeta{\lambda_{t-1}})$, then
$\bestRsqr[\lambda_{t-1}]{\beta}{\theta}=0$, $\theta=(y-X\beta)/\lambda_{t-1}$
and
\begin{align*}
\bestRsqr{\beta}{\theta} &= \left(
\frac{\lambda_{t-1}^2}{\lambda_t^2} (1-\frac{\lambda_{t}}{\lambda_{ t-1 } } )
- (\frac{\lambda_{t-1}}{\lambda_{t}}-1 ) \right) \norm{\theta}^2,\nonumber\\
&\leq\left(   \frac{1}{\lambda_t} -
\frac{1}{\lambda_{t-1}}\right)^2 \norm{y}^2,
\end{align*}
since  $\|\theta\|\leq \|y\|/\lambda_{t-1}$ for $\theta=\ttheta{\lambda_{t-1}}$.

Note that contrarily to former sequential rules \cite{Wang_Zhou_Wonka_Ye13},
our introduced GAP SAFE rules still work when one has only access to
approximations of $\ttheta{\lambda_{t-1}}$.
\end{remark}



\section{Experiments}\label{sec:experiments}
\subsection{Coordinate Descent}\label{subsec:coordinate_descent}

Screening procedures can be used with any optimization algorithm. We chose
coordinate descent because it is well suited for machine learning tasks,
especially with sparse and/or unstructured design matrix $X$.
Coordinate descent requires to extract efficiently
columns of $X$ which is typically not
easy in signal processing applications where $X$ is commonly
an implicit operator (e.g. Fourier or wavelets).

\renewcommand{\algorithmicloop}{
\textbf{every $f$ passes through the data}}

\begin{algorithm}
\caption{Coordinate descent with GAP SAFE rules}
\begin{algorithmic}
\INPUT{$X, y , \epsilon, K, f , (\lambda_t)_{t \in [T-1]}$}
\STATE{Initialization:
\STATE \quad $\lambda_0=\lambda_{\max}$
\STATE \quad $\beta^{\lambda_0}=0$
}
\FOR{$ t \in [T-1]$}
\STATE $\beta \leftarrow \beta^{\lambda_{t-1}}$ (previous $\epsilon$-solution)
\FOR{$k \in [K]$}
\IF{$ k\mod f = 1 $}
\STATE Compute $\theta$ and $\mathcal{C}$ thanks to \eqref{eq:thetak}
and \eqref{eq:GAP_SAFE_SPHERE} or \eqref{eq:GAP_SAFE_DOME}
\STATE Get $A^{\lambda_t}(\mathcal{C})= \{j \in [p] :
\mu_{\mathcal{C}}(x_j) \geq 1\}$ as in \eqref{eq:active_set}
\IF{$G_{\lambda_t}(\beta,\theta) \leq \epsilon$}
\STATE $\beta^{\lambda_t}\leftarrow \beta$
\STATE {\bf break}
\ENDIF
\ENDIF
\FOR{$j \in A^{\lambda_t}(\mathcal{C})$}
\STATE
$\beta_{j} \leftarrow
\mathrm{ST}\big(\frac{\lambda_t}{\norm{x_j}^2}, \beta_{j}-\frac{
x_j^\top (X \beta - y)}{\norm{x_j}^2}\big)$
\STATE
\# $\mathrm{ST}(u, x) = \sign(x)\left(|x|- u \right)_+$ (soft-threshold)
\ENDFOR
\ENDFOR
\ENDFOR
\OUTPUT{$(\beta^{\lambda_t})_{t\in [T-1]}$}
\end{algorithmic}
\label{alg:cd_screening}
\end{algorithm}

We implemented the screening rules of Table~\ref{table:spheres} based
on the coordinate descent in Scikit-learn~\cite{Pedregosa_etal11}.
This code is written in Python and Cython to generate low level C
code, offering high performance. A low level language is
necessary for this algorithm to scale. Two implementations
were written to work efficiently with both dense data stored
as Fortran ordered arrays and sparse data stored in the
compressed sparse column (CSC) format.
Our pseudo-code is presented in Algorithm~\ref{alg:cd_screening}. In practice,
we perform the dynamic screening
tests every $f=10$ passes through the entire (active) variables. Iterations are
stopped when the duality gap is smaller than the target accuracy.

The naive computation of $\theta_k$ in
\eqref{eq:thetak} involves the computation of $\norm{X^\top
\rho_k}_{\infty}$ ($\rho_k$ being the current residual), which
costs $\mathcal{O}(np)$ operations. This can
be avoided as one knows when using a safe rule
that the index achieving the maximum for this norm is in
$A^{\lambda_t}(\mathcal{C})$. Indeed, by construction $\argmax_{j \in
A^{\lambda_t}(\mathcal{C})} |x_j^\top\theta_k|=\argmax_{j\in[p]}
|x_j^\top\theta_k|=\argmax_{j\in[p]} |x_j^\top\rho_k|$. In practice
the evaluation of the dual gap is therefore not a $\mathcal{O}(np)$ but
$\mathcal{O}(nq)$ where $q$ is the size of $A^{\lambda_t}(\mathcal{C})$.
In other words, using screening also speeds up the evaluation
of the stopping criterion.

We did not compare our method against the strong rules of
\citet{Tibshirani_Bien_Friedman_Hastie_Simon_Tibshirani12}
because they are not safe and therefore need complex post-processing with
parameters to tune. Also we did not compare against the sequential rule of
\citet{Wang_Zhou_Wonka_Ye13} (\eg EDDP) because it requires the exact dual optimal solution
of the previous Lasso problem, which is not available in practice and can
prevent the solver from actually converging: this is a phenomenon we always observed on our experiments.

\subsection{Number of screened variables}


\begin{figure}
\centering
\includegraphics[width=1\linewidth]{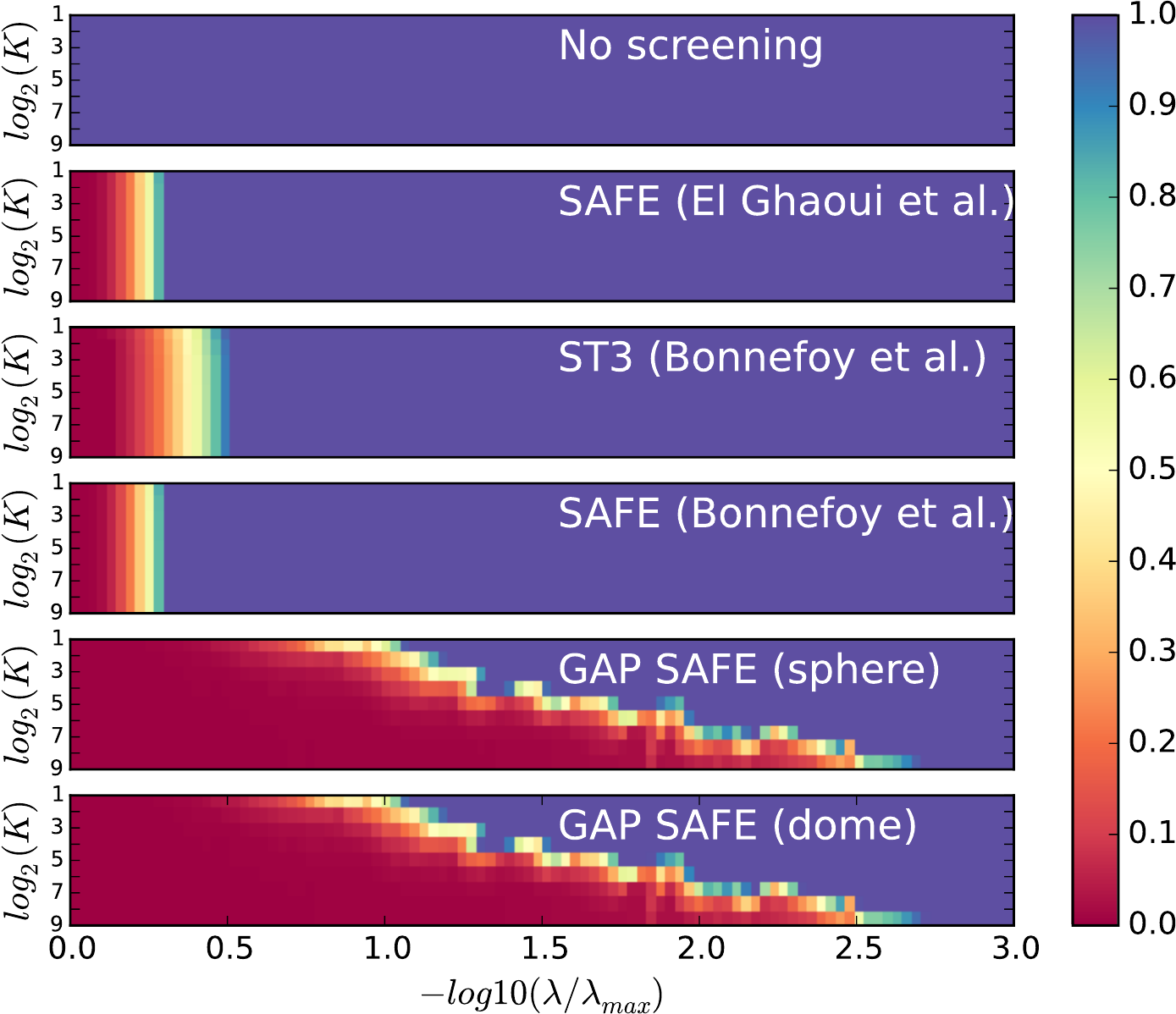}
\caption{Proportion of active variables as a function of $\lambda$ and the
number of iterations $K$ on the Leukemia dataset. Better strategies have longer
range of $\lambda$ with (red) small active sets.}
\label{fig:screening_proportion}
\end{figure}

Figure~\ref{fig:screening_proportion} presents the proportion of variables
screened by several safe rules on the standard Leukemia dataset.
The screening proportion is presented as a function of the number of iterations
$K$. As the SAFE screening rule of \citet{ElGhaoui_Viallon_Rabbani12} is
sequential but not dynamic, for a given $\lambda$ the proportion of screened
variables does not depend on $K$.
The rules of \citet{Bonnefoy_Emiya_Ralaivola_Gribonval14} are more efficient
on this dataset but they do not benefit much from the dynamic framework. 
Our proposed GAP SAFE tests screen much more variables,
especially when the tuning parameter $\lambda$ gets small, which is particularly
relevant in practice. Moreover, even for very small $\lambda$'s
(notice the logarithmic scale) where no variable is screened at the beginning of
the optimization procedure, the GAP SAFE rules manage to screen more
variables, especially when $K$ increases. Finally, the figure demonstrates that
the GAP SAFE dome test only brings marginal improvement over the sphere.


\begin{figure}
\centering
\includegraphics[width=1.00\linewidth]{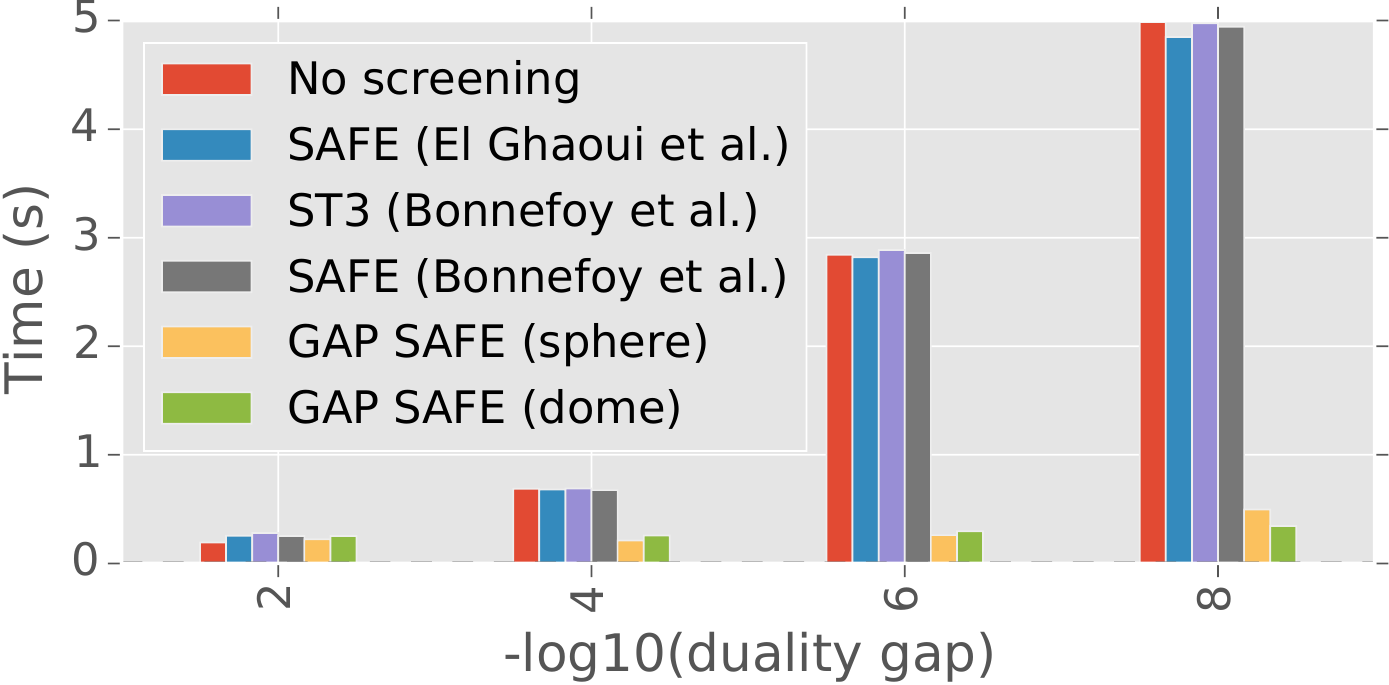}
\caption{Time to reach convergence using various screening
rules on the Leukemia dataset (dense data: $n=72, p=7129$).
}
\label{fig:leukemia_gaps_bars}
\end{figure}

\subsection{Gains in the computation of Lasso paths}

The main interest of variable screening is to reduce computation costs.
Indeed, the time to compute the screening itself should not be larger than
the gains it provides. Hence, we compared the time needed to
compute Lasso paths to prescribed accuracy for different safe rules.
Figures \ref{fig:leukemia_gaps_bars}, \ref{fig:news_gaps_bars}
and~\ref{fig:rcv1_gaps_bars}
illustrate results on three datasets. Figure~\ref{fig:leukemia_gaps_bars} presents
results on the dense, small scale, Leukemia dataset.
Figure~\ref{fig:news_gaps_bars} presents results on a medium scale sparse
dataset obtained with bag of words features extracted from the 20newsgroup dataset (comp.graphics vs. talk.religion.misc with TF-IDF
removing English stop words and words occurring only once or more than 95\% of the time).
Text feature extraction was done using Scikit-Learn.
Figure~\ref{fig:rcv1_gaps_bars}
focuses on the large scale sparse RCV1 (Reuters Corpus Volume 1) dataset, \lcf \cite{Schmidt_LeRoux_Bach13}.

In all cases, Lasso paths are computed as required to estimate
optimal regularization parameters in practice (when using cross-validation
one path is computed for each fold). For each Lasso path, solutions are
obtained for $T=100$ values of $\lambda$'s, as detailed in
Section~\ref{sec:sequential_safe_rules}.
Remark that the grid used is the default one in both Scikit-Learn and the
glmnet R package.
With our proposed GAP SAFE screening we obtain on all datasets substantial gains in computational time. We can already
get an up to 3x speedup when we require a duality gap smaller than
$10^{-4}$. The interest of the screening is even clearer for higher accuracies:
GAP SAFE sphere is 11x faster than its competitors on the Leukemia
dataset, at accuracy $10^{-8}$. One can observe that with the parameter grid
used here, the larger is $p$ compared to $n$, the higher is the gain
in computation time.

In our experiments, the other safe screening rules did not show much speed-up.
As one can see on Figure~\ref{fig:screening_proportion}, those screening rules
keep all the active variables for a wide range of $\lambda$'s. The algorithm is
thus faster for large $\lambda$'s but slower afterwards, since we still compute
the screening tests. Even if one can avoid some of these useless
computations thanks to formulas like \eqref{eq:lambda_critic_dynamic} or
\eqref{eq:lambda_critic_static}, the corresponding speed-up would not be
significant. 


\begin{figure}
\centering
\includegraphics[width=1.00\linewidth]{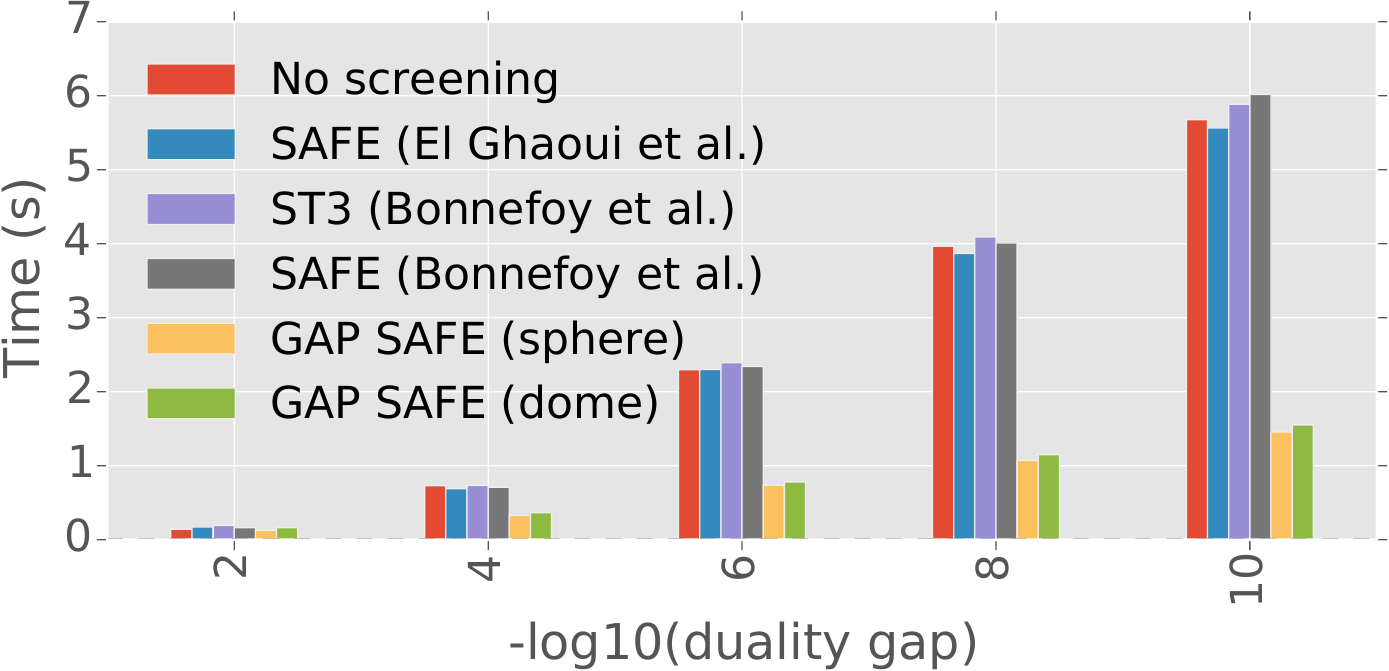}
\caption{Time to reach convergence using various screening
rules on bag of words from the 20newsgroup dataset (sparse data: with
$n=961, p=10094$).
}
\label{fig:news_gaps_bars}
\end{figure}

\begin{figure}
\centering
\includegraphics[width=1.00\linewidth]{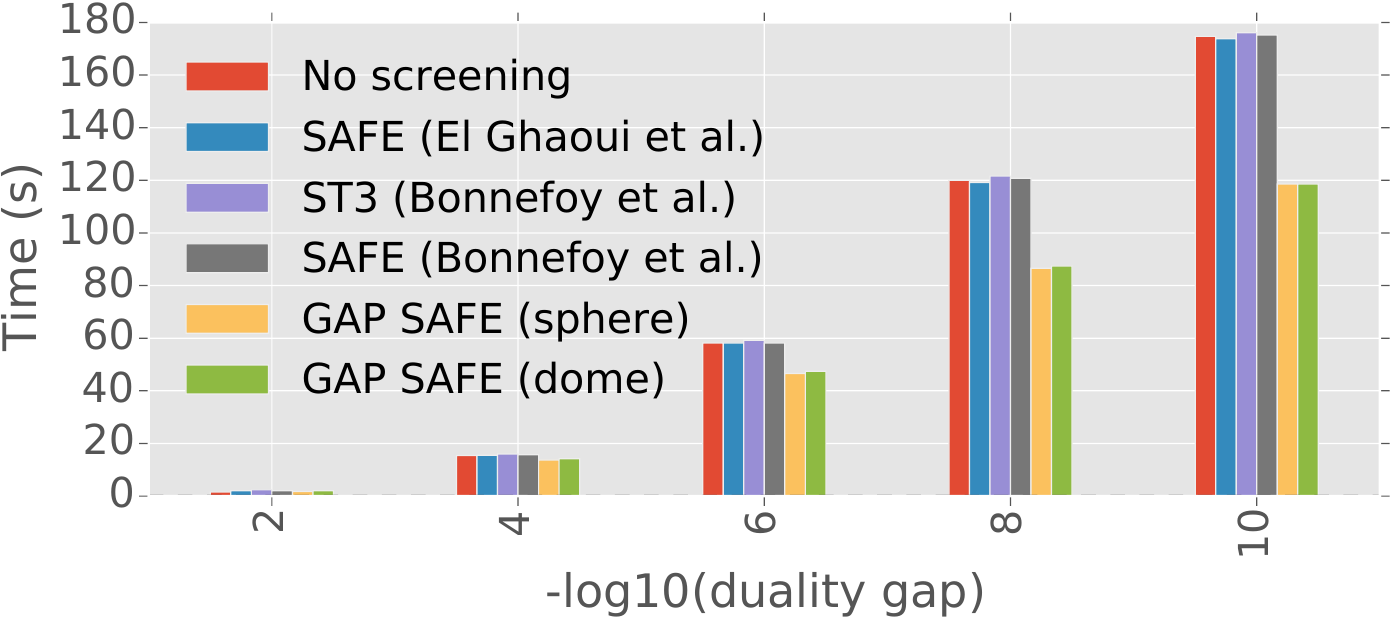}
\caption{Computation time to reach convergence using different screening
strategies on the RCV1 (Reuters Corpus Volume 1) dataset (sparse data with
$n=20242$ and $p=47236$).
}
\label{fig:rcv1_gaps_bars}
\end{figure}



\section{Conclusion}
We have presented new results on safe rules
for accelerating algorithms solving the Lasso problem
(see Appendix for extension to the Elastic Net).
First, we have introduced the framework of converging safe rules, a key concept
independent of the implementation chosen.
Our second contribution was to leverage duality gap computations to create
two safer rules satisfying the aforementioned convergence properties.
Finally, we demonstrated the important practical benefits of those new
rules by applying them to standard dense and sparse datasets using a
coordinate descent solver.
Future works will extend our
framework to generalized linear model and group-Lasso.

\section*{Acknowledgment}
The authors would like to thanks Jalal Fadili and Jingwei Liang for helping clarifying some misleading statements on the equicorrelation set.
We acknowledge the support from Chair Machine Learning for Big Data
at T\'el\'ecom ParisTech and from the Orange/T\'el\'ecom ParisTech
think tank phi-TAB. This work benefited from the support of the
"FMJH Program Gaspard Monge in optimization and operation
research", and from the support to this program from EDF.

\bibliography{references_all}
\bibliographystyle{icml2015}

\appendix
\newpage
\onecolumn

\section{Supplementary materials}
We provided in this Appendix some more details on the theoretical results given in the main part.

\subsection{Dome test}

Let us consider the  case where the safe region $\mathcal{C}$ is the dome
$\mathcal{D}(\domecenter,\domeradius,\domeprop,\domenormal)$, with parameters:
center $\domecenter$, radius $\domeradius$, relative distance ratio $\alpha$
and unit normal vector $\domenormal$.

The computation of the dome test formula proceeds as follows:

\begin{align}\label{eq:mu_dome}
 \sigma_{\mathcal{C}}(x_j)=
\begin{cases}
c^\top x_j + \domeradius \|x_j\| & \text{ if } \domenormal^\top x_j< -\alpha
\|x_j \|, \\
c^\top x_j - \domeradius \alpha  \domenormal^\top x_j +\domeradius \sqrt{(\|x_j\|^2-|\domenormal^\top x_j|^2)(1-\alpha^2)} & \text{ otherwise. }
\end{cases}
\end{align}
and so
\begin{align}\label{eq:mu_dome2}
 \sigma_{\mathcal{C}}(-x_j)=
\begin{cases}
-c^\top x_j + \domeradius \|x_j\| & \text{ if } -\domenormal^\top x_j< -\alpha
\|x_j \|, \\
-c^\top x_j + \domeradius \alpha  \domenormal^\top x_j +\domeradius \sqrt{(\|x_j\|^2-|\domenormal^\top x_j|^2)(1-\alpha^2)} & \text{ otherwise. }
\end{cases}
\end{align}
With the previous display we can now compute
$\mu_{\mathcal{C}}(x_j):=\max( \sigma_{\mathcal{C}}(x_j),\sigma_{\mathcal{C}}(-x_j))$. Thanks to the
Eq.~\eqref{eq:sphere_test},  we express our dome test as:
\begin{equation}
\text{ If } \, M_{\min} <c^\top x_j < M_{\max}, \, \text{ then } \, \tbeta{\lambda}_j=0.
\end{equation}

Using the former notation:

\begin{equation}
 M_{\max}=
\begin{cases}
1 - \domeradius \|x_j\| & \text{ if } \domenormal^\top x_j< -\alpha \|x_j\|, \\
1 + \domeradius \alpha  \domenormal^\top x_j -\domeradius \sqrt{(\|x_j\|^2-|\domenormal^\top x_j|^2)(1-\alpha^2)} & \text{ otherwise. }
\end{cases}
\end{equation}

\begin{equation}
 M_{\min}=
\begin{cases}
-1 + \domeradius \|x_j\| & \text{ if } -\domenormal^\top x_j< -\alpha \|x_j\|,
\\
-1 + \domeradius \alpha  \domenormal^\top x_j +\domeradius \sqrt{(\|x_j\|^2-|\domenormal^\top x_j|^2)(1-\alpha^2)} & \text{ otherwise. }
\end{cases}
\end{equation}

Let us introduce the following dome parameters, for any $\theta \in \dual$:
\begin{itemize}
 \item Center: $ \domecenter= (y/\lambda+\theta)/2$.
 \item Radius: $ \domeradius= \largeR[\theta]/2$.
 \item Ratio: $ \alpha= -1+2\smallR[\theta]^2/\largeR[\theta]^2$.
 \item Normal vector: $ \domenormal=(y/\lambda - \theta)/\largeR[\theta]$.
\end{itemize}

Reminding that the support function of a set
is the same as the support function of its closed convex hull \cite{Hiriart-Urruty_Lemarechal93}[Proposition V.2.2.1] means
that we only need to optimize over the dome introduced.
Therefore, one cannot improve our previous result by optimizing the problem on the intersection of the ball of radius
$\largeR$ and the complement of the ball of radius $\smallR$ (\ie the blue region in Figure
\ref{fig:couronne1}).

\subsection{Proof of Theorem~\ref{th:converging_safe}}\label{sec:proofs0}
\begin{proof}
Define $\max_{j \notin \mathcal{E}_{\lambda}}
| x_j^{\top} \ttheta{\lambda} |= t < 1$.
Fix $\epsilon>0$ such that $\epsilon<(1-t)/ (\max_{j \notin
\mathcal{E}_{\lambda}} \|x_j\|)$.
As $\mathcal{C}_k$ is a converging sequence containing $\ttheta{\lambda}$, its
diameter is converging to zero, and there exists $k_0 \in \bbN$ such that
$\forall k\geq k_0, \forall \theta \in \mathcal{C}_k, \|\theta
-\ttheta{\lambda}\|\leq \epsilon$. Hence, for any $j
\notin \mathcal{E}_{\lambda}$ and any $\theta \in \mathcal{C}_k$,
$|x_j^\top(\theta -\ttheta{\lambda})|\leq  (\max_{j \notin \mathcal{E}_\lambda}
\|x_j\|) \|\theta -\ttheta{\lambda}\| \leq
(\max_{j \notin \mathcal{E}_\lambda} \|x_j\|) \epsilon$.
Using the triangle inequality, one gets
\begin{align*}
|x_j^\top \theta|\leq & (\max_{j \notin \mathcal{E}_\lambda} \|x_j\|)
\epsilon + \max_{j \notin \mathcal{E}_\lambda } | x_j^{\top} \ttheta{\lambda} |
\\
\leq & (\max_{j \notin \mathcal{E}_\lambda} \|x_j\|) \epsilon+t< 1,
\end{align*}
provided that $\epsilon<(1-t)/ (\max_{j \notin \mathcal{E}_\lambda}
\|x_j\|)$. Thus, for any $k \geq k_0, \mathcal{E}_{\lambda}^c \subset
Z^{(\lambda)} (\mathcal{C}_k) = A^{(\lambda)} (\mathcal{C}_k)^c$ and $
A^{(\lambda)} (\mathcal{C}_k)
\subset \mathcal{E}_{\lambda}$.

For the reverse inclusion take $j\in \mathcal{E}_{\lambda}$, \ie
$| x_j^{\top} \ttheta{\lambda} | = 1$.
Since for all $k \in \bbN, \ttheta{\lambda} \in \mathcal{C}_{k}$, then
$j \in A^{(\lambda)}(\mathcal{C}_k)=\{ j \in [p] : \max_{\theta \in
\mathcal{C}_{k}} | x_j^{\top} \theta |  \geq 1 \}$ and the result holds.

\end{proof}

\subsection{Proof of Proposition~\ref{prop:sequential}}\label{sec:proofs}

We detail here the proof of Proposition~\ref{prop:sequential}.
\begin{proof}
We first use the fact that
$$G_{\lambda_{t-1}} (\beta,\theta)=\frac{1}{2}
\norm{X \beta -y}^2_2+\lambda_{t-1}\norm{\beta}_1 -
 \frac{1}{2}\norm{y}_2^2 + \frac{\lambda_{t-1}^2}{2}
\norm{\theta-\frac{y}{\lambda_{t-1}}}^2_2, $$
to obtain
\begin{equation*}
\norm{\beta}_1 = \frac{1}{\lambda_{t-1}}
\Big( \frac{1}{2}\norm{y}_2^2 - \norm{X \beta -y}^2_2
- \frac{\lambda_{t-1}^2}{2} \norm{\theta-\frac{y}{\lambda_{t-1}}}^2_2
+G_{\lambda_{t-1}} (\beta,\theta)\Big).
\end{equation*}
Then,
\begin{align*}
G_{\lambda_{t}} (\beta,&\theta) = \frac{1}{2}\norm{X \beta -y}^2_2 + \frac{\lambda_{t}}{\lambda_{t-1}}
\Big( \frac{1}{2}\norm{y}_2^2 - \frac{1}{2} \norm{X \beta -y}^2_2- \frac{\lambda_{t-1}^2}{2}
\norm{\theta-\frac{y}{\lambda_{t-1}}}^2_2
+G_{\lambda_{t-1}} (\beta,\theta)\Big) \\
& \qquad \qquad
 - \frac{1}{2}\norm{y}^2_2 + \frac{\lambda_t^2}{2} \norm{\theta - \frac{y}{\lambda_{t}}}^2_2 \\
= & \frac{1}{2}(\frac{\lambda_{t}}{\lambda_{t-1}}-1)\norm{y}_2^2 + \frac{1}{2}(1-\frac{\lambda_{t}}{\lambda_{t-1}}) \norm{X \beta -y}^2_2
+ \frac{\lambda_{t}}{\lambda_{t-1}}G_{\lambda_{t-1}} (\beta,\theta)
+\frac{1}{2}\big( \norm{\lambda_t \theta -y}^2_2 - \frac{\lambda_{t}}{\lambda_{t-1}}\norm{\lambda_{t-1} \theta -y}^2_2\big)\\
= &\frac{1}{2}(\frac{\lambda_{t}}{\lambda_{t-1}}-1)\norm{y}_2^2
+\frac{1}{2}(1-\frac{\lambda_{t}}{\lambda_{t-1}}) \norm{X \beta -y}^2_2+ \frac{\lambda_{t}}{\lambda_{t-1}}G_{\lambda_{t-1}} (\beta,\theta)\\
& \qquad \qquad+\frac{1}{2}\Big( \norm{\lambda_t  \theta -y}^2_2 -
\frac{\lambda_{t}}{\lambda_{t-1}}\big(\norm{\lambda_t  \theta
-y}^2_2+\norm{(\lambda_{t-1}-\lambda_t )\theta}_2^2 + 2 ( \lambda_t  \theta
-y)^\top (\lambda_{t-1}-\lambda_t )\theta \big)\Big).
\end{align*}
We deal with the dot product as
\[
2 \lambda_t  (\lambda_{t-1}-\lambda_t ) (\theta -\frac{y}{\lambda_{t}})^\top
\theta =\lambda_t  (\lambda_{t-1}-\lambda_t ) \big(
\norm{\theta}^2_2+\norm{\theta-\frac{y}{\lambda_{t}}}^2_2 -
\norm{\frac{y}{\lambda_{t}}}^2_2\big).
\]
Hence,
\begin{align*}
G_{\lambda_{t}} (\beta,\theta)  = &
\frac{1}{2}(\frac{\lambda_{t}}{\lambda_{t-1}}-1+\frac{1}{\lambda_{t-1}}(\lambda_{t-1}-\lambda_t ))\norm{y}_2^2
+ \frac{1}{2}(1-\frac{\lambda_{t}}{\lambda_{t-1}}) \norm{X \beta -y}^2_2 \\
&-\frac{\lambda_t}{2}(\lambda_{t-1}-\lambda_t ))\norm{\theta}^2_2 +\frac{1}{2}( 1 - \frac{\lambda_{t}}{\lambda_{t-1}} - \frac{1}{\lambda_{t-1}}(\lambda_{t-1}-\lambda_t ) )\norm{\lambda_t  \theta -y}^2_2+\frac{\lambda_{t}}{\lambda_{t-1}}G_{\lambda_{t-1}} (\beta,\theta) \\
=& \frac{1}{2}\left(1-\frac{\lambda_{t}}{\lambda_{t-1}}\right) \norm{X \beta -y}^2_2
- \frac{\lambda_t}{2}(\lambda_{t-1}-\lambda_t )) \| \theta\|^2+
\frac{\lambda_{t}}{\lambda_{t-1}}G_{\lambda_{t-1}} (\beta,\theta).\\
\end{align*}
We observe in the end that
\begin{align*}
\frac{2}{\lambda_t^2}G_{\lambda_{t}} (\beta,\theta)  = &
\left(1-\frac{\lambda_{t}}{\lambda_{t-1}}\right)
 \norm{\frac{X \beta -y}{\lambda_t}}^2_2  -\left(\frac{\lambda_{t-1}}{\lambda_{t}}-1 \right)\norm{\theta}^2_2
+ \frac{2}{\lambda_{t-1}\lambda_t}G_{\lambda_{t-1}} (\beta,\theta).
\end{align*}
\qedhere
\end{proof}

\subsection{Elastic-Net}
The previously proposed tests can be adapted straightforwardly to the
Elastic-Net estimator \cite{Zou_Hastie05}. We provide here some more details
for the interested reader.
\begin{equation}\label{eq:enet}
\min_{\beta \in \bbR^p} \frac{1}{2} \norm{X \beta - y}_2^2
+ \lambda  \alpha  \norm{\beta}_1
+ \frac{\lambda}{2}(1 - \alpha)\norm{\beta}_2^2.
\end{equation}
One can reformulate this problem as a Lasso problem
\begin{equation}\label{eq:larger}
\min_{\beta \in \bbR^{p}} \frac{1}{2} \norm{\tilde{X} \beta - \tilde{y}}_2^2 +
\lambda \alpha \norm{\beta}_1,
\end{equation}
where $\tilde{X}=\begin{pmatrix}
                 X\\
		\sqrt{(1-\alpha)\lambda} I_p
                 \end{pmatrix}\in \bbR^{n+p,p}
$
and $\tilde{y}=\begin{pmatrix}
                y\\
 0
               \end{pmatrix} \in \bbR^{n+p}.
$
With this modification all the tests introduced for the Lasso can be
adapted for the Elastic-Net.

\end{document}